\documentclass[11pt, a4paper, copyright]{google}

\usepackage[authoryear, sort&compress, round]{natbib}
\usepackage{url}
\usepackage{graphicx}
\usepackage{subcaption}
\usepackage{booktabs}
\usepackage[table]{xcolor}
\usepackage{xspace}
\usepackage{algorithm}
\usepackage{algpseudocode}
\usepackage{amsmath}
\usepackage{amssymb}
\usepackage{amsthm}
\usepackage[
  breaklinks,
  colorlinks,
  linkcolor=violet,
  urlcolor=myblue,
  citecolor=purple
]{hyperref}

\definecolor{myblue}{rgb}{0.21,0.49,0.74}

\usepackage{listings}
\usepackage{tcolorbox}
\tcbuselibrary{listings,skins,breakable}

\usepackage{amsmath,amsfonts,bm}

\def\eqref#1{equation~\ref{#1}}

\def\1{\bm{1}}

\DeclareMathAlphabet{\mathsfit}{\encodingdefault}{\sfdefault}{m}{sl}
\SetMathAlphabet{\mathsfit}{bold}{\encodingdefault}{\sfdefault}{bx}{n}

\newtheorem{assumption}{Assumption}
\newtheorem{proposition}{Proposition}
\newtheorem{theorem}{Theorem}
\newtheorem{remark}{Remark}

\newcommand{\autoenv}{\textsc{GenEnv}\xspace}
\newcommand{\dataevolving}{\textsc{Data-Evolving Paradigm}\xspace}
\newcommand{\alphacurr}{\emph{$\alpha$-Curriculum Reward}\xspace}
\newcommand{\finding}[1]{\noindent\textbf{Finding.} #1}

\definecolor{lightblue}{RGB}{230,240,255}

\uselogo{}

\title{GenEnv: Difficulty-Aligned Co-Evolution Between LLM Agents and Environment Simulators}

\author{
\centering
{\bfseries
Jiacheng Guo\textsuperscript{1*},
Ling Yang\textsuperscript{1*$\dagger$},
Peter Chen\textsuperscript{2*},
Qixin Xiao\textsuperscript{3*}, \\
\textbf{Yinjie Wang\textsuperscript{4}},
\textbf{Xinzhe Juan\textsuperscript{3}},
\textbf{Jiahao Qiu\textsuperscript{1}},
\textbf{Ke Shen},
\textbf{Mengdi Wang\textsuperscript{1$\dagger$}}
\par}

{\small\normalfont\mdseries
\begin{tabular}{c}
\textsuperscript{1}Princeton University \quad
\textsuperscript{2}Columbia University \quad
\textsuperscript{3}University of Michigan \quad
\textsuperscript{4}University of Chicago
\end{tabular}
\par}
{\small\normalfont\mdseries
\makebox[\linewidth][c]{\textsuperscript{*}Equal Contribution \quad \textsuperscript{$\dagger$}Corresponding Authors}
}
}

\begin{document}

\begin{abstract}
\vspace{-2em}
Training capable Large Language Model (LLM) agents is critically bottlenecked by the high cost and static nature of real-world interaction data. We address this by introducing \textbf{\autoenv}, a framework that establishes a \textbf{difficulty-aligned co-evolutionary game} between an agent and a scalable, generative environment simulator. Unlike traditional methods that evolve models on static datasets, \autoenv instantiates a \textbf{\dataevolving}: the simulator acts as a dynamic curriculum policy, continuously generating tasks specifically tailored to the agent's ``zone of proximal development''. This process is guided by a simple but effective \textbf{$\alpha$-Curriculum Reward}, which aligns task difficulty with the agent's current capabilities. We evaluate \autoenv on five benchmarks, including API-Bank, ALFWorld, BFCL, Bamboogle, and TravelPlanner. Across these tasks, \autoenv improves agent performance by up to \textbf{+40.3\%} over 7B baselines and matches or exceeds the average performance of larger models. Compared to Gemini 2.5 Pro-based offline data augmentation, \autoenv achieves better performance while using \textbf{3.3$\times$ less data}. By shifting from static supervision to adaptive simulation, \autoenv provides a data-efficient pathway for scaling agent capabilities. Our codes
are available at \url{https://github.com/Gen-Verse/GenEnv}.
\end{abstract}

\maketitle

\begin{figure}[h]
    \vspace{+2em}
    \centering
    \includegraphics[width=\textwidth]{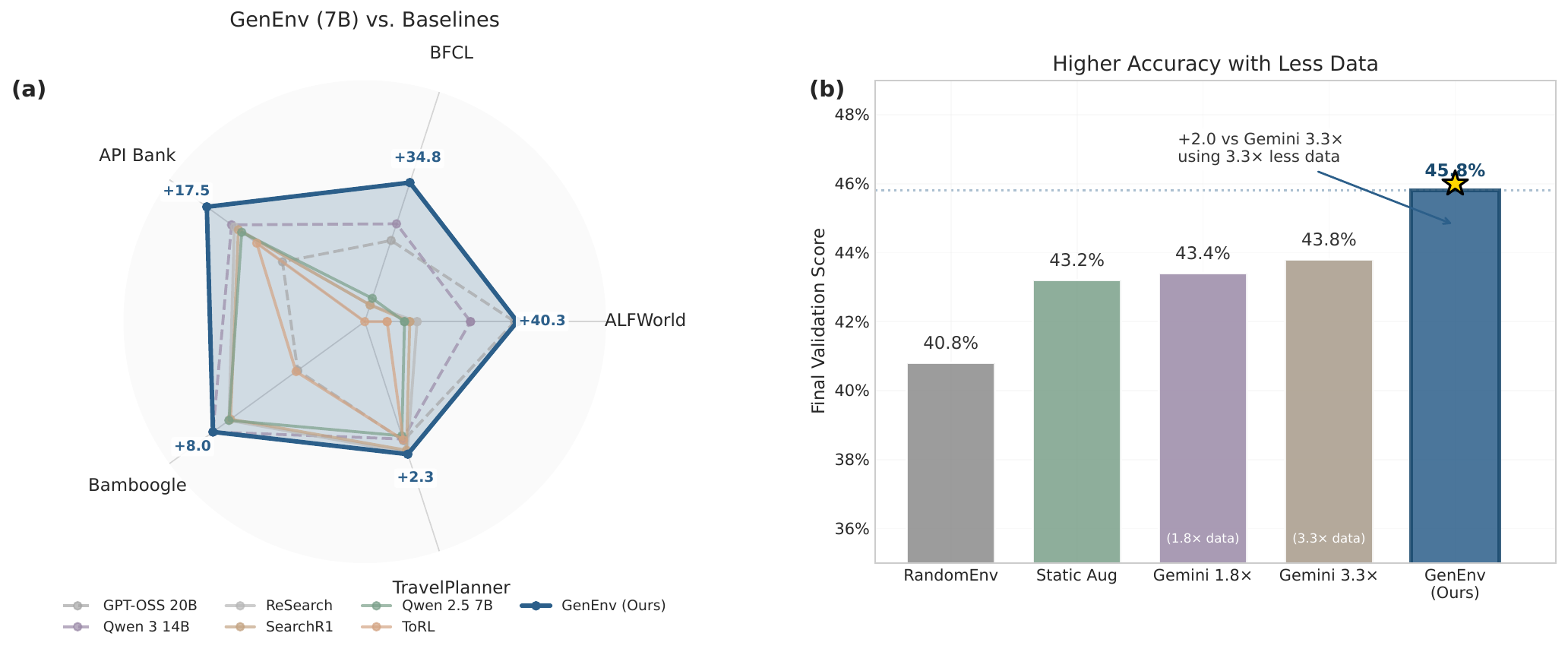}
    \caption{
\textbf{GenEnv’s cross-benchmark gains and data efficiency.}
(a) We compare \autoenv (7B) against representative baselines (Qwen2.5-7B, ReSearch, SearchR1, ToRL) and larger open models (e.g., Qwen3-14B, GPT-OSS-20B).
Blue callouts report the \emph{absolute} improvement of \autoenv over Qwen2.5-7B on each benchmark.
(b) Validation data-efficiency comparison on BFCL: \autoenv surpasses RandomEnv and Static Augmentation, and outperforms Gemini-based offline augmentation even with 3.3$\times$ more synthetic data.
Together, the figure shows that \emph{difficulty-aligned adaptive simulation} can outperform stronger static augmentation baselines under comparable training settings.
}

    \label{fig:fig0_overview}
    \vspace{-0.5em}
\end{figure}

\section{Introduction}
\label{sec:intro}
Training capable Large Language Model (LLM) agents for complex, interactive tasks like web navigation or tool use is constrained by a significant bottleneck: the high cost of data collection through real-world interaction \citep{yao2022, shinn2023, ning2025, wang2025co,wang2025cure}. Each step an agent takes in a real-world environment can be slow, expensive, and difficult to parallelize. For instance, a web agent that navigates an e-commerce site may fail when a button's label changes from ``Add to Cart'' to ``Add to Basket'' \citep{gur2023}, but discovering such failure modes requires extensive and costly real-world exploration. This fragility highlights a key limitation in how these agents are commonly trained.

A central driver of this issue is the reliance on static, pre-collected datasets of expert trajectories \citep{pomerleau1991, levine2020, samadi2024}. Such datasets, no matter how large, represent a fixed snapshot of the world and struggle to capture the wide range of variations an agent will encounter in open-ended environments \citep{levine2020}. Increasing the dataset size alone does not resolve this limitation: the bottleneck often lies not just in data volume, but in the static and costly nature of its collection and its inability to adapt as the agent improves.

\begin{figure}[t]
\centering
\includegraphics[width=\textwidth]{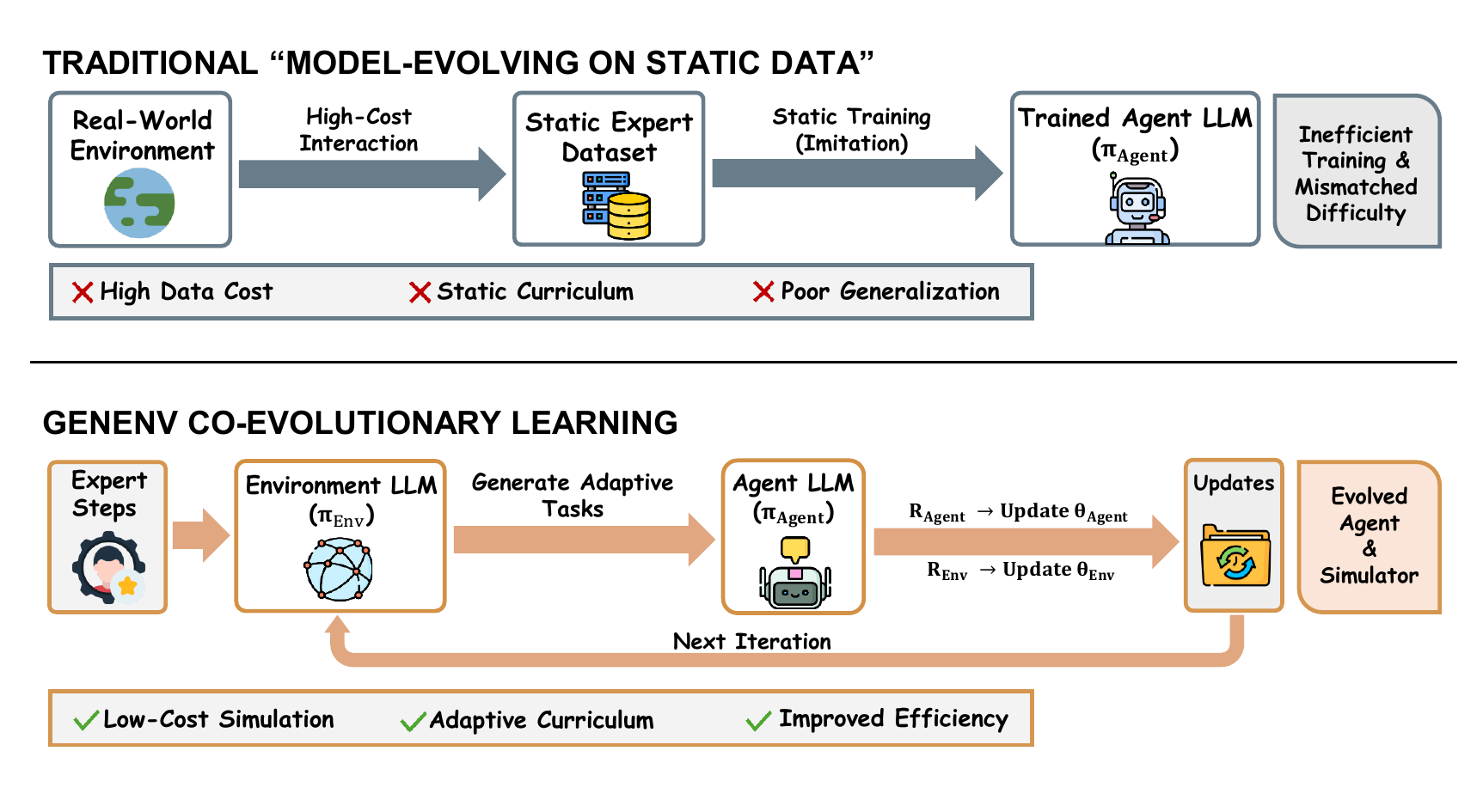}
\caption{A comparison between the traditional training paradigm and our proposed \autoenv framework. The traditional approach (top) relies on high-cost interaction with the real world to create a static dataset, leading to inefficient training and poor generalization. \autoenv (bottom) creates a co-evolutionary loop where an Environment LLM generates adaptive tasks for the Agent LLM, enabling low-cost simulation, an adaptive curriculum, and improved efficiency.}
\label{fig:fig1_overview}
\end{figure}

The challenge of insufficient and static data has led to significant interest in synthetic data generation. However, despite progress, these methods often produce a large but ultimately static corpus that can fail to adapt to the agent's evolving requirements \citep{ye2024, ding2024}. This can result in inefficient training that still does not effectively target the agent's specific weaknesses. The high cost of interaction and data collection remains a core problem.

To address this, we approach the problem differently by proposing \textbf{\autoenv}, a framework built on leveraging an LLM as a \textbf{scalable environment simulator} to reduce interaction costs. Instead of relying on slow and expensive real-world feedback, our framework trains the agent almost entirely within a simulated environment that can generate diverse and relevant training scenarios at a substantially lower cost. As illustrated in Figure~\ref{fig:fig1_overview}, this contrasts with traditional methods that evolve a model on static data. In our approach, a generative environment model is trained to produce an adaptive curriculum of tasks, creating challenges tailored to the agent's performance. This leads to our primary research question: \textit{Can an LLM-based environment simulator provide a scalable, low-cost alternative to real-world interaction for effectively training capable agents?}

In this simulation-centric process, the agent learns to overcome challenges generated by the simulator. The agent's performance provides a natural reward signal that guides the simulator's curriculum generation, allowing both to improve in a self-contained training loop. Throughout the paper, we use ``agent'' and \textbf{Agent Policy} $\pi_{\text{agent}}$ interchangeably, and ``environment simulator'' and \textbf{Environment Policy} $\pi_{\text{env}}$ interchangeably. As previewed in Figure~\ref{fig:fig0_overview}, \autoenv delivers consistent gains over strong 7B baselines across five benchmarks and achieves higher accuracy than Gemini-based offline augmentation while using less synthetic data, highlighting the advantage of difficulty-aligned, adaptive simulation over static scaling of data. Our \textbf{contributions} are:

\begin{itemize}
    \item \textbf{The Data-Evolving Paradigm:} We propose a co-evolutionary framework where the training data distribution adapts dynamically to the agent's learning progress, breaking the reliance on static corpora.
    \item \textbf{Difficulty-Aligned Simulation:} We introduce the \alphacurr, a mechanism that rewards the simulator for generating tasks within the agent's target success zone (akin to the ``zone of proximal development''~\citep{vygotsky1978mind}), ensuring an efficient automated curriculum.
    \item \textbf{Data Efficiency:} On our benchmarks, \autoenv matches or surpasses Gemini 2.5 Pro-based static augmentation pipelines while using \textbf{3.3$\times$ less synthetic data}, suggesting that an adaptive simulator can be more valuable than simply scaling the teacher model.
\end{itemize}

\section{\autoenv: Difficulty-Aligned Co-Evolution}
\label{sec:methodology}

\autoenv views agent training as a \emph{two-player curriculum game} rather than a single-player optimization problem. We maintain two policies: an \textbf{Agent Policy $\pi_{\text{agent}}$} (the agent) and an \textbf{Environment Policy $\pi_{\text{env}}$} (the environment simulator). Unlike standard RL where the environment is fixed, \autoenv enables both to co-evolve:

\begin{itemize}
    \item $\pi_{\text{agent}}$ learns to solve tasks sampled from the current simulator.
    \item $\pi_{\text{env}}$ is rewarded for generating tasks whose difficulty is \textit{aligned} with the agent's current capability---targeting the ``zone of proximal development'' where learning is most effective~\citep{vygotsky1978mind}.
\end{itemize}

\subsection{Data-Evolving Paradigm: From Static Corpora to Adaptive Simulation}
Standard training minimizes a loss $\mathcal{L}(\theta)$ over a static distribution $\mathcal{D}_{\text{static}}$, where $\theta$ denotes the parameters of the agent. In contrast, \autoenv implements a \dataevolving{}. The training data $\mathcal{D}_t$ is generated on-the-fly by $\pi_{\text{env}}$, conditioned on the agent's historical performance. This creates a feedback loop (Figure~\ref{fig:fig2_loop}) where the simulator seeks not to defeat the agent, but to find its ``breaking points'' to facilitate learning.

\begin{figure}[t]
\centering
\includegraphics[width=\textwidth]{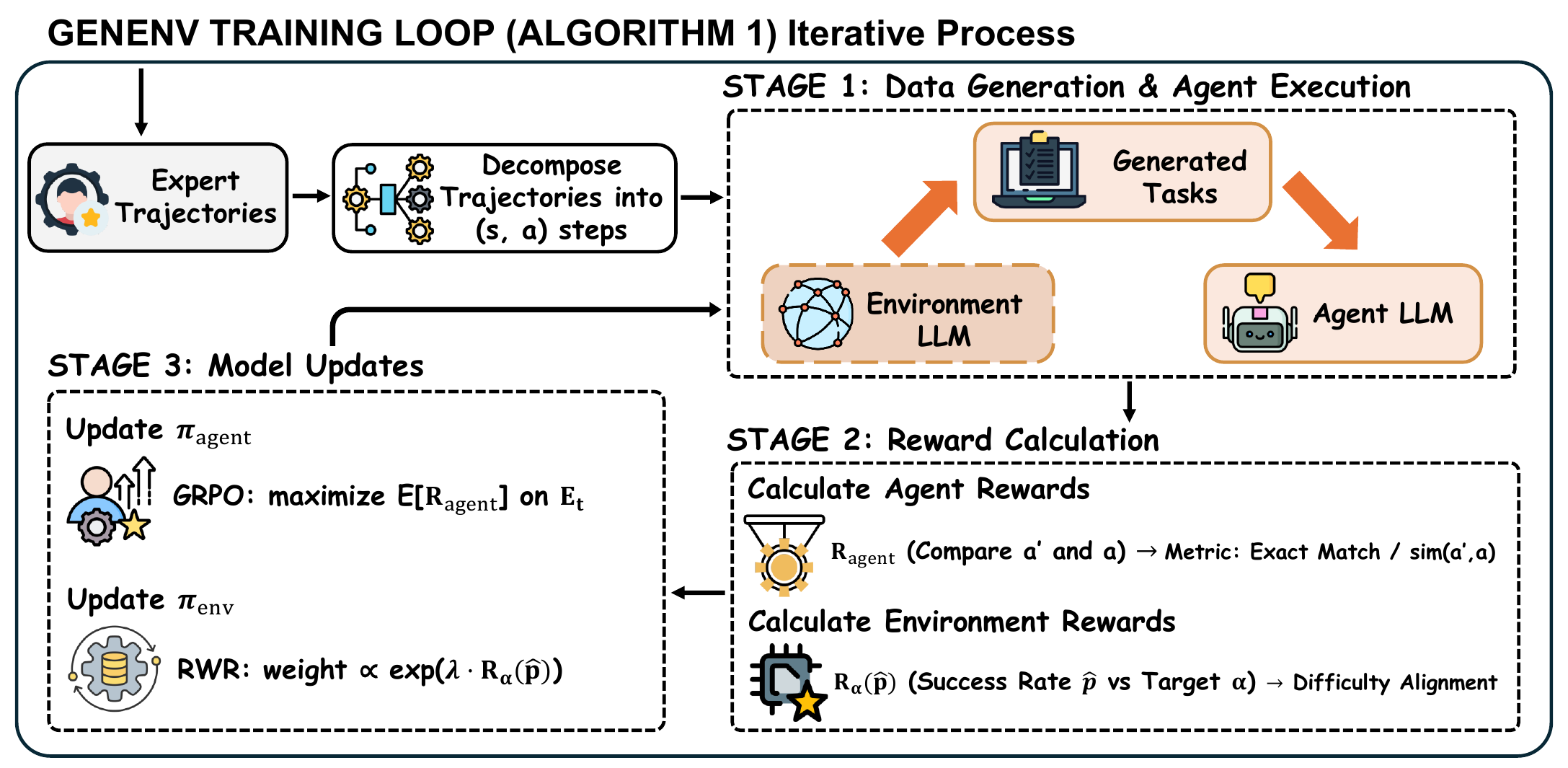}
\caption{\textbf{The \autoenv Co-Evolutionary Loop.} (1) The Environment Policy generates tasks. (2) The Agent Policy attempts them. (3) The environment reward (difficulty alignment) updates the simulator, while the agent reward (task success) updates the agent.}
\label{fig:fig2_loop}
\end{figure}

\subsection{Rewards: Agent vs.\ Environment (with explicit equation references)}
\label{sec:rewards}

\subsubsection{Agent Task Reward ($R_{\text{agent}}$)}
Each environment-generated task induces a target action/trajectory $a$ (e.g., a sequence of tool calls or a final answer), and the agent produces a prediction $a'$. We distinguish a structured action space $\mathcal{A}_{\text{struct}}$ (e.g., executable API calls) from free-form answers (e.g., natural language). For structured actions we can rely on exact execution; for unstructured ones we use a soft similarity score.
We define the agent reward (used to update $\pi_{\text{agent}}$) as:
\begin{equation}
\label{eq:r_agent}
    R_{\text{agent}}(a', a)
    =
    \mathbb{I}(a' = a)\cdot \mathbb{I}(a \in \mathcal{A}_{\text{struct}})
    +
    \text{sim}(a', a)\cdot \mathbb{I}(a \notin \mathcal{A}_{\text{struct}}),
\end{equation}
where $\text{sim}(a', a) \in [0,1]$ is task-dependent (e.g., normalized token-F1 or embedding similarity). In all benchmarks we scale $R_{\text{agent}}$ into $[0,1]$.

\subsubsection{Environment Difficulty-Alignment Reward ($R_{\text{env}}$)}
The core innovation for $\pi_{\text{env}}$ is a difficulty-aligned reward that targets a \emph{success-rate band} around a desired $\alpha \in (0,1)$ (we use $\alpha=0.5$). For each generated \emph{batch} of $n$ task variations, after the agent attempts them we compute the empirical success rate:
\begin{equation}
\label{eq:p_hat}
    \hat{p} = \frac{k}{n},
\end{equation}
where $k$ is the number of successes under $R_{\text{agent}}$ (Eq.~\eqref{eq:r_agent}). We then assign the environment reward (used to update $\pi_{\text{env}}$):
\begin{equation}
\label{eq:r_env}
    R_{\text{env}}(\hat{p})
    =
    \exp\!\left(-\beta(\hat{p}-\alpha)^2\right),
\end{equation}
where $\beta>0$ controls sharpness. This bell-shaped reward peaks when the agent's performance matches $\alpha$, discouraging tasks that are already mastered ($\hat{p}\to 1$) or hopeless ($\hat{p}\to 0$).
We additionally apply a \emph{difficulty filter}: task batches with $|\hat{p}-\alpha|>k_{\text{min}}$ (we use $k_{\text{min}}=0.1$) are excluded from environment updates to prevent overfitting to transient spikes.

\subsection{Data Structures and How New Training Data Is Produced}
\label{sec:data_aggregation}

A recurring confusion in co-evolution papers is \emph{where the training data actually comes from}. We therefore make the data flow explicit.

\paragraph{What the environment generates.}
At epoch $t$, the environment policy $\pi_{\text{env}}$ generates a task batch $\mathcal{T}_t$. Concretely, $\mathcal{T}_t$ contains $n$ \emph{task instances} (often multiple variations of the same seed), where each instance includes:
(i) a task prompt/context (including tool specs, constraints, and goal),
(ii) an evaluation specification (e.g., executable checker / exact-match target / reference answer),
and optionally (iii) structured ``ground truth'' target action $a$ when the benchmark provides it (e.g., tool-call arguments).

\paragraph{What the agent produces.}
The agent interacts with each task instance and yields an interaction trace (rollout) that we denote by an element $e \in \mathcal{E}_t$.
Each trace records at minimum:
\[
e = (\text{task}, \text{trajectory}, a', r),
\]
where \texttt{trajectory} can include intermediate reasoning text and tool calls, $a'$ is the final output/action, and $r=R_{\text{agent}}(a',a)$ is computed via Eq.~\eqref{eq:r_agent} (or its benchmark-specific instantiation).

\paragraph{Two growing datasets: one for the agent, one for the environment.}
We maintain two pools that grow online:
\begin{itemize}
    \item \textbf{Agent training pool $\mathcal{D}_{\text{train}}$:} stores \emph{valid} interaction traces from $\mathcal{E}_t$. A trace is ``valid'' if it is well-formed and evaluable for the benchmark (e.g., tool calls parse and execute; outputs follow required schema; checker runs without error). We append these valid tuples to $\mathcal{D}_{\text{train}}$ so the agent can (i) learn from fresh on-policy experiences and (ii) retain mastery of earlier curricula by continuing to sample from the accumulated pool.
    \item \textbf{Environment SFT pool $\mathcal{D}_{\text{env}}$:} stores environment generations used to train $\pi_{\text{env}}$ via RWR. Each record is a supervised pair of the form
    \[
        (\text{env-conditioning context} \rightarrow \text{generated task instance}),
    \]
    where the conditioning context includes the seed prompt plus any summary signals (e.g., recent success statistics) that $\pi_{\text{env}}$ conditions on. We weight each record by a monotone function of the environment reward in Eq.~\eqref{eq:r_env} (e.g., $\propto \exp(\lambda R_{\text{env}}(\hat{p}))$, where $\lambda=1.0$ is a temperature hyperparameter).
\end{itemize}

\paragraph{How this produces a ``data-evolving'' training set.}
At every epoch, both $\mathcal{T}_t$ and $\mathcal{E}_t$ are newly generated; thus $\mathcal{D}_{\text{train}}$ is not a fixed offline corpus but an evolving mixture of (a) base data, (b) previously collected valid traces, and (c) newly collected on-policy traces.
Meanwhile, $\mathcal{D}_{\text{env}}$ evolves toward generating tasks whose empirical success rate stays near $\alpha$ (Eq.~\eqref{eq:r_env}), which in turn shifts the difficulty distribution of future $\mathcal{T}_{t+1}$.
This closes the loop: \emph{new data is produced as a byproduct of interaction}, and is then explicitly \emph{aggregated} into training pools for subsequent updates.

\subsection{Two-Player Curriculum RL: Optimization Loop}
\label{sec:opt_loop}

\paragraph{Player 1 Update (Agent).} The Agent Policy $\pi_{\text{agent}}$ is updated to maximize $\mathbb{E}[R_{\text{agent}}]$ (Eq.~\eqref{eq:r_agent}). In our experiments, we instantiate this using Group Relative Policy Optimization (GRPO) \citep{shao2024deepseekmath}.
\vspace{-1em}
\paragraph{Player 2 Update (Environment).} The Environment Policy $\pi_{\text{env}}$ is updated to maximize $\mathbb{E}[R_{\text{env}}]$ (Eq.~\eqref{eq:r_env}). We implement this via \textbf{Reward-Weighted Regression} (RWR): from each environment-generated batch, we compute $\hat{p}$ (Eq.~\eqref{eq:p_hat}) from agent rollouts and assign $R_{\text{env}}(\hat{p})$. We then construct a weighted SFT set of environment generations and fine-tune $\pi_{\text{env}}$ toward higher-reward generations. For stability, we regularize updates with a KL penalty to the initial simulator and cap per-step updates by a maximum KL threshold.

\begin{algorithm}
\caption{\textsc{\autoenv Co-Evolutionary Loop} (with explicit equation references)}
\label{alg:ces}
\begin{algorithmic}[t]
\State \textbf{Initialize:} Agent $\pi_{\text{agent}}$, Environment $\pi_{\text{env}}$, Agent pool $\mathcal{D}_{\text{train}}$, Env pool $\mathcal{D}_{\text{env}}$.
\For{epoch $t = 1, \dots, T$}
    \State \textbf{\(\rhd\) Phase 1: \texttt{Online Generation \& Interaction}}
    \State Environment generates a task batch $\mathcal{T}_t \sim \pi_{\text{env}}(\cdot)$.
    \State Agent $\pi_{\text{agent}}$ rolls out on $\mathcal{T}_t$ to obtain traces $\mathcal{E}_t$ and per-trajectory agent rewards $R_{\text{agent}}$ via \eqref{eq:r_agent}.
    \State Compute batch success $\hat{p}$ via \eqref{eq:p_hat} and assign environment reward $R_{\text{env}}(\hat{p})$ via \eqref{eq:r_env}.
    
    \State \textbf{\(\rhd\) Phase 2: \texttt{Dual Update (Two Players, Two Objectives)}}
    \State Update \textbf{agent} $\pi_{\text{agent}}$ via GRPO to maximize $\mathbb{E}[R_{\text{agent}}]$ (\eqref{eq:r_agent}).
    \State Filter out batches with $|\hat{p}-\alpha|>k_{\text{min}}$ for environment updates.
    \State Build weighted env SFT set $\tilde{\mathcal{D}}_{\text{env}}^t$ from $\mathcal{T}_t$ with weights $\propto \exp(\lambda R_{\text{env}}(\hat{p}))$ (\eqref{eq:r_env}).
    \State Update \textbf{environment} $\pi_{\text{env}}$ via RWR on $\tilde{\mathcal{D}}_{\text{env}}^t$ to maximize $\mathbb{E}[R_{\text{env}}]$ (\eqref{eq:r_env}).
    
    \State \textbf{\(\rhd\) Phase 3: \texttt{Aggregation (How New Data Enters Training)}}
    \State Extract valid traces from $\mathcal{E}_t$ (e.g., parseable/executable/checker-passed) and append to agent pool:
    \State \hspace{1.2em} $\mathcal{D}_{\text{train}} \leftarrow \mathcal{D}_{\text{train}} \cup \text{Valid}(\mathcal{E}_t)$.
    \State Append weighted environment generations to env pool:
    \State \hspace{1.2em} $\mathcal{D}_{\text{env}} \leftarrow \mathcal{D}_{\text{env}} \cup \tilde{\mathcal{D}}_{\text{env}}^t$.
\EndFor
\end{algorithmic}
\end{algorithm}

\section{Theoretical Analysis of Difficulty-Aligned \autoenv}
\label{sec:theory}

In this section we provide a simple theoretical analysis of the
difficulty-aligned co-evolution mechanism in \autoenv. Our goal is not
to fully characterize the dynamics of large LLMs, but to clarify why
(1) tasks whose success rate is close to the target band $\alpha$
carry the strongest learning signal for the Agent Policy
$\pi_{\text{agent}}$, and (2) the $\alpha$-Curriculum Reward
$R_{\text{env}}$ provides a statistically consistent signal for the
Environment Policy $\pi_{\text{env}}$ to rank task types by how well
their difficulty matches the current agent.

\subsection{Intermediate Difficulty Maximizes Agent Learning Signal}
\label{sec:theory-gradient}

We first consider a stylized bandit setting in which the Agent Policy
$\pi_{\text{agent}}$ interacts with a single environment-generated
task type $\tau$ (e.g., a family of API-calling problems of similar
difficulty). The outcome of each attempt is a scalar reward
$r \in \{0,1\}$, where $r=1$ denotes success on the task and $r=0$
denotes failure.\footnote{The analysis extends to
$R_{\text{agent}} \in [0,1]$ by rescaling; we use the binary case for
clarity.}
For a fixed Agent Policy parameterization $\theta$, let
$p(\tau) = \Pr(r=1 \mid \tau, \theta)$ denote the success probability
on task type $\tau$. The Agent Policy is updated with a REINFORCE-style estimator
\begin{equation}
    g(\tau, r)
    = (r - b(\tau)) \nabla_\theta \log \pi_\theta(a \mid \tau),
    \label{eq:reinforce}
\end{equation}
where $a$ is the sampled action (e.g., a rollout trajectory of tool-calling
sequence) and $b(\tau)$ is a baseline (e.g., an estimate of the
expected reward on $\tau$).
The quantity $\mathbb{E}[\|g(\tau, r)\|^2]$ can be viewed as measuring
how strong the stochastic gradient signal is for this task type.

Given the trust-region KL constraint and the gradient-clipping bias in GRPO-style policy updates, it is expected that the squared norm of the score function remains within a trust-region bound. This behavior has been discussed in recent analyses of one-step policy updates (e.g., \citet[Theorem 3.2]{chen2025exploration}) as well as in the literature on trust-region policy optimization \citep{schulman2015trust}. Accordingly, we establish our theoretical analysis under the reasonable assumption that the squared norm of the score function $\nabla_\theta \log \pi_\theta(a \mid \tau)$ does not vary too dramatically when conditioned on the binary outcome $r$.

\begin{assumption}[Bounded score variation]
\label{assump:bounded-score}
For a fixed task type $\tau$, there exist constants
$0 < c_{\min} \le c_{\max} < \infty$ such that
\[
    c_{\min}
    \;\le\;
    \mathbb{E}\!\left[
        \big\| \nabla_\theta \log \pi_\theta(a \mid \tau) \big\|^2
        \,\middle|\, r
    \right]
    \;\le\;
    c_{\max}
    \quad
    \text{for both } r=0 \text{ and } r=1.
\]
\end{assumption}

We take the baseline to be the on-task expected reward,
$b(\tau) = \mathbb{E}[r \mid \tau] = p(\tau)$, which is the variance
minimizer in the standard REINFORCE analysis. Under these conditions we
obtain the following result.

\begin{proposition}[Intermediate difficulty maximizes gradient signal]
\label{prop:info}
Suppose Assumption~\ref{assump:bounded-score} holds and the baseline is
chosen as $b(\tau) = p(\tau)$. Then there exist positive constants
$C_{\min}$ and $C_{\max}$, independent of $p(\tau)$, such that
\begin{equation}
    C_{\min} \, p(\tau)\big(1 - p(\tau)\big)
    \;\le\;
    \mathbb{E}\big[\|g(\tau, r)\|^2\big]
    \;\le\;
    C_{\max} \, p(\tau)\big(1 - p(\tau)\big).
    \label{eq:variance-bound}
\end{equation}
In particular, up to constant factors, the expected squared gradient
norm is proportional to $p(\tau)\big(1 - p(\tau)\big)$, which is
maximized when $p(\tau) = 1/2$, i.e., for tasks of intermediate
difficulty.
\end{proposition}

\begin{proof}[Proof sketch]
With $r \in \{0,1\}$ and $b(\tau) = p(\tau)$, we have
$\mathbb{E}[(r - p(\tau))^2 \mid \tau] = \mathrm{Var}(r \mid \tau)
:= p(\tau)(1 - p(\tau))$. Using the law of total expectation and
Assumption~\ref{assump:bounded-score}, we can factor out the variation
coming from the score function up to multiplicative constants, which
yields~\eqref{eq:variance-bound}. The function $p(1-p)$ is a concave
quadratic on $[0,1]$ with a unique maximum at $p=1/2$.
A full proof is given in Appendix~\ref{app:theory-proofs}.
\end{proof}

\begin{remark}[$\frac12$-Curriculum reward] Considering the case that $\alpha = \tfrac{1}{2}$, we have the identity
\(
p(1-p) = \tfrac{1}{4} - \bigl(p - \tfrac{1}{2}\bigr)^2 .
\) Thus, maximizing the variance term $p(1-p)$ is exactly equivalent to minimizing the squared distance to the target success rate
$\alpha = \tfrac{1}{2}$. In \autoenv, the Environment Policy does not observe
the true success probability $p(\tau)$ but only an empirical estimate
$\hat{p}(\tau)$ from a finite number of rollouts. The
$\alpha$-Curriculum Reward takes the form
\begin{equation}
    R_{\text{env}}(\hat{p}(\tau))
    = \exp\!\big(\!-\beta (\hat{p}(\tau) - \alpha)^2\big),
\end{equation}
which is a monotone transformation of
$-(\hat{p}(\tau) - \alpha)^2$ and therefore encourages the simulator
to propose tasks whose empirical success rate stays close to $\alpha$.
Proposition~\ref{prop:info} then suggests that, in expectation, this
aligns the simulator with task types that provide the strongest
learning signal for $\pi_{\text{agent}}$.\end{remark}

\subsection{Ranking Consistency of the $\alpha$-Curriculum Reward}
\label{sec:theory-ranking}

We next show that, despite relying on noisy empirical success rates,
the $\alpha$-Curriculum Reward provides a statistically consistent
signal for ranking task types by how well their difficulty matches
the target band. The argument is based on standard concentration
inequalities.

Consider two task types $\tau_1$ and $\tau_2$. For a fixed Agent Policy
$\pi_{\text{agent}}$, let $p_i = p(\tau_i)$ denote the true success
probability on $\tau_i$, and define their distances to the target band
$\alpha$ as
\begin{equation}
    \Delta_i = |p_i - \alpha|, \quad i \in \{1,2\}.
\end{equation}
Without loss of generality, assume $\Delta_1 < \Delta_2$, i.e.,
$\tau_1$ is closer to the target difficulty than $\tau_2$.
For each $\tau_i$ we run $n_i$ independent rollouts and compute the
empirical success rate $\hat{p}_i = k_i / n_i$, where $k_i$ is the
number of successes. The Environment Policy receives the reward
\begin{equation}
    R_{\text{env}}(\hat{p}_i)
    = \exp\big(-\beta(\hat{p}_i - \alpha)^2\big).
\end{equation}
Since the exponential is monotone, ranking tasks by $R_{\text{env}}$ is
equivalent to ranking them by their squared distance
$(\hat{p}_i - \alpha)^2$ to the target band.

The following theorem shows that the mis-ranking probability decays
exponentially in the minimum number of rollouts.

\begin{theorem}[Ranking consistency of $R_{\text{env}}$]
\label{thm:ranking}
Let $n = \min\{n_1, n_2\}$ and
$\Delta_1 < \Delta_2$ as above. Define
$\delta = (\Delta_2 - \Delta_1)/3 > 0$.
Then
\begin{equation}
    \Pr\big(R_{\text{env}}(\hat{p}_1) \le R_{\text{env}}(\hat{p}_2)\big)
    \;\le\;
    4 \exp\!\left(
        - \frac{2}{9} (\Delta_2 - \Delta_1)^2 \, n
    \right).
\end{equation}
In particular, as $n \to \infty$ the reward ranking is consistent:
tasks whose true success probability lies closer to the target band
$\alpha$ receive higher $\alpha$-Curriculum Reward with probability
approaching $1$ at an exponential rate.
\end{theorem}

\begin{proof}[Proof sketch]
Because the exponential is monotone,
$R_{\text{env}}(\hat{p}_1) \le R_{\text{env}}(\hat{p}_2)$ is equivalent to
$|\hat{p}_1 - \alpha| \ge |\hat{p}_2 - \alpha|$.
We show that if both empirical estimates $\hat{p}_i$ lie within
$\delta$ of their true means $p_i$, then necessarily
$|\hat{p}_1 - \alpha| < |\hat{p}_2 - \alpha|$ and hence
$R_{\text{env}}(\hat{p}_1) > R_{\text{env}}(\hat{p}_2)$.
Thus a mis-ranking can only occur when at least one empirical mean
deviates from its expectation by more than $\delta$, which can be
bounded using Hoeffding's inequality for Bernoulli random variables.
A detailed proof is provided in Appendix~\ref{app:theory-proofs}.
\end{proof}

\paragraph{Implications for \autoenv.}
Theorem~\ref{thm:ranking} shows that, even though the Environment
Policy only observes noisy empirical success rates $\hat{p}_i$ derived
from a finite number of rollouts, the $\alpha$-Curriculum Reward is a
statistically consistent proxy for task difficulty. As we increase the
rollout budget per task type, the environment LLM can more reliably
identify and up-weight task families whose difficulty lies in the
target zone of proximal development. This provides a formal
justification for the empirical convergence behaviour observed in
Figure~\ref{fig:difficulty_convergence}, where the agent's success rate on
simulated tasks concentrates around a band centered at $\alpha = 0.5$.
Together, Proposition~\ref{prop:info} and Theorem~\ref{thm:ranking}
clarify why aligning the simulator with an intermediate success rate
both maximizes learning signal for the agent and yields a stable,
difficulty-calibrated curriculum.

\begin{table}[t]
\centering
\caption{\textbf{Main Results.} Comparison on five benchmarks. Models are grouped by size: \textbf{Large Scale Models} ($>10$B) are sorted by size descending, followed by \textbf{7B Models}. Bold numbers indicate the best performance within each group. \textbf{GenEnv (7B)} significantly outperforms other 7B baselines and even surpasses the average performance of several 72B/405B models on this suite.}
\resizebox{0.9\textwidth}{!}{%
\begin{tabular}{lcccccc}
\toprule
\textbf{Model} & \textbf{ALFWorld} & \textbf{BFCL} & \textbf{API-Bank} & \textbf{Bamboogle} & \textbf{TravelPlanner} & \textbf{Average} \\
\midrule
\multicolumn{7}{l}{\textit{Large Scale Models}} \\
\midrule
Llama 3.1 405B & \textbf{65.3} & 5.5 & \textbf{74.4} & \textbf{77.6} & 16.5 & 47.9 \\
GPT-OSS 120B & 60.4 & 21.9 & 53.6 & 29.6 & 14.7 & 36.0 \\
Qwen 2.5 72B & 63.5 & \textbf{35.3} & 54.9 & 69.6 & 20.5 & \textbf{48.8} \\
Llama 3.1 70B & 60.1 & 13.4 & 64.3 & 76.8 & 17.6 & 46.4 \\
Qwen 3 32B & 52.3 & 33.8 & 63.8 & 71.2 & \textbf{22.5} & 48.7 \\
GPT-OSS 20B & 53.6 & 24.4 & 41.2 & 33.6 & 14.9 & 33.5 \\
Qwen 3 14B & 37.8 & 29.4 & 66.7 & 76.0 & 14.7 & 44.9 \\
\midrule[1.5pt]
\multicolumn{7}{l}{\textit{7B Models}} \\
\midrule
ReSearch & 18.7 & 5.0 & 65.3 & 68.0 & 16.4 & 34.7 \\
SearchR1 & 16.1 & 5.0 & 63.3 & 67.2 & 16.1 & 33.5 \\
Qwen 2.5 7B & 14.2 & 7.0 & 61.6 & 68.0 & 14.3 & 33.0 \\
ToRL & 8.0 & 0.0 & 54.1 & 34.4 & 14.8 & 22.3 \\
\rowcolor{lightblue} \textbf{GenEnv (Ours)} & \textbf{54.5} & \textbf{41.8} & \textbf{79.1} & \textbf{76.0} & \textbf{16.6} & \textbf{53.6} \\
\bottomrule
\end{tabular}
}

\label{tab:results}
\end{table}
\section{Experiments}
\label{sec:experiments}

\subsection{Experimental Setup}
\paragraph{Backbone Models.}
Unless otherwise specified, all 7B agents and simulators are initialized from Qwen2.5-7B-Instruct~\citep{yang2024qwen2_5}. For large-scale baselines in Table~\ref{tab:results}, we include Llama 3.1 models~\citep{meta2024llama3}, GPT-OSS open-weight models~\citep{openai2025gptoss}, and Qwen3 models~\citep{yang2025qwen3}. These models are evaluated using the same tool-calling interface and prompt templates as our 7B baselines for fairness.
\vspace{-1em}
\paragraph{Benchmarks.} We evaluate across 5 diverse benchmarks that span tool use, embodied interaction, and real-world planning. 
\textbf{API-Bank}~\citep{li2023apibank} measures function-calling and tool-augmented reasoning, and \textbf{Bamboogle} is a compositional multi-hop QA benchmark built on top of the framework from \citet{press2022measuring}; for both, we follow the evaluation protocols from ToRL~\citep{li2025torl}. 
\textbf{ALFWorld}~\citep{shridhar2020alfworld} aligns textual instructions with embodied environments; we utilize the official validation set, where multi-turn tasks are decomposed into single steps for evaluation. 
\textbf{BFCL} follows the Berkeley Function-Calling Leaderboard setup~\citep{patil2024berkeley}; specifically, we evaluate on the long-context subset treating each turn independently. 
\textbf{TravelPlanner}~\citep{xie2024travelplanner} captures end-to-end planning and tool use in realistic travel scenarios; we report the average of four metrics: CS Micro (\%), CS Macro (\%), HD Micro (\%), and HD Macro (\%). The \textbf{Average} column in Table~\ref{tab:results} is the unweighted mean of these per-benchmark success metrics.
\vspace{-0.5em}
\paragraph{Baselines \& Variants.} We compare against standard instructed models (Qwen2.5-7B-Instruct) and specialized search-and-planning agents, including ReSearch, SearchR1, and ToRL~\citep{jin2025searchr1, li2025torl}. These methods represent strong model-evolving pipelines that either improve search policies or alignment rewards on largely static datasets. To strictly evaluate our Data-Evolving contribution, we define:
\vspace{-0.5em}
\begin{itemize}
    \item \textbf{GenEnv-Random:} The simulator generates new tasks every epoch but is \textit{not} trained via $R_{\text{env}}$. This isolates the effect of dynamic data vs.\ aligned curriculum.
    \item \textbf{GenEnv-Static:} The simulator generates a large batch of synthetic data once before training.
    \item \textbf{Gemini-Offline (2x / 3.3x):} High-quality synthetic data generated offline by Gemini 2.5 Pro (approx.\ 1.76x and 3.27x the training set size). This represents a strong ``teacher-distillation'' baseline.
\end{itemize}
\begin{figure}[t]
\centering
\includegraphics[width=0.95\textwidth]{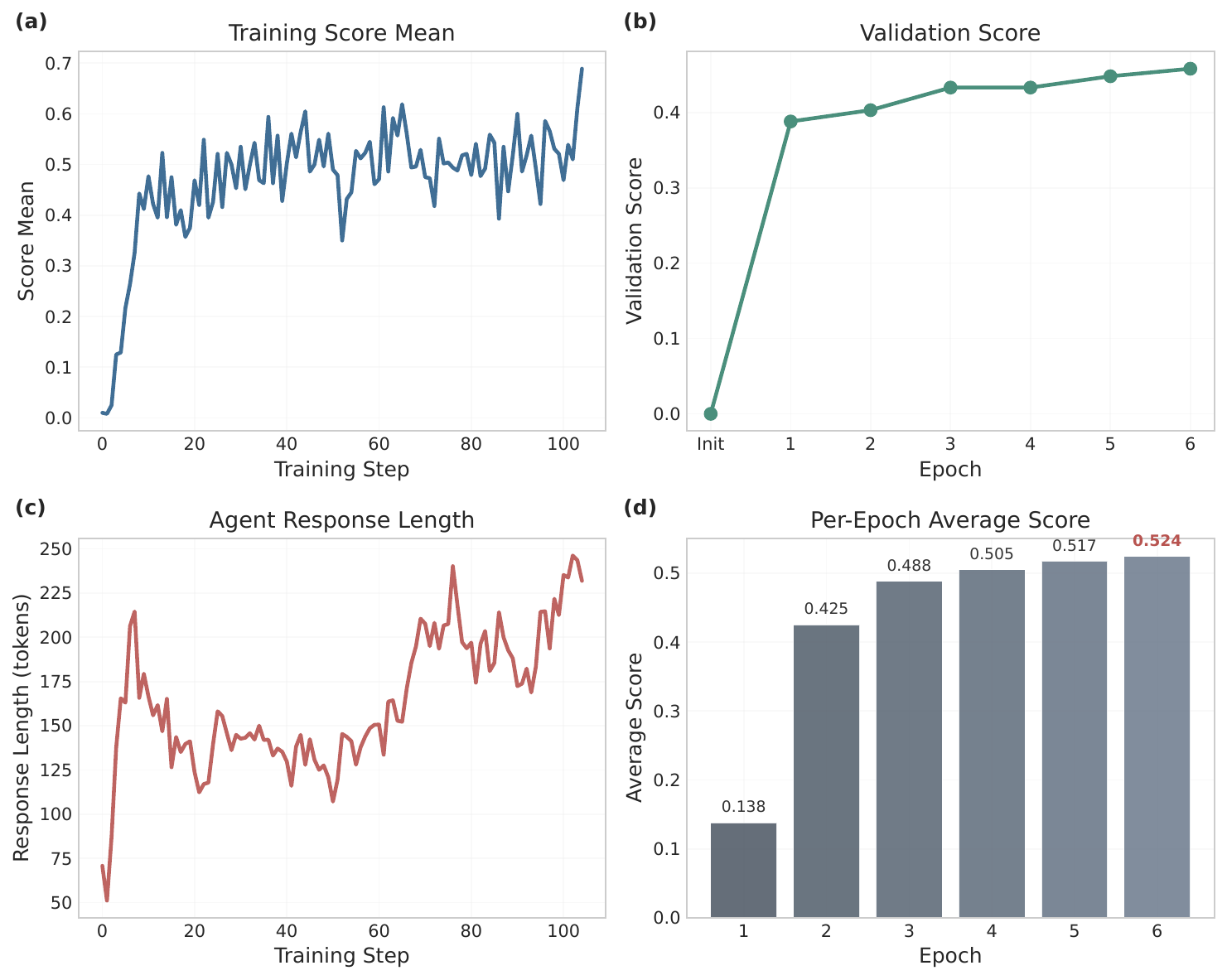}
\caption{\textbf{Training dynamics of \autoenv.} From left to right: (a) training step-wise reward (\texttt{critic/score/mean}); (b) validation score across epochs; (c) batch-level ground-truth accuracy; and (d) per-epoch average reward. The curves show that \autoenv trains stably without reward collapse or divergence, with both reward and accuracy improving smoothly over time.}
\label{fig:training_stability}
\end{figure}

\begin{figure}[h!]
\centering
\includegraphics[width=\textwidth]{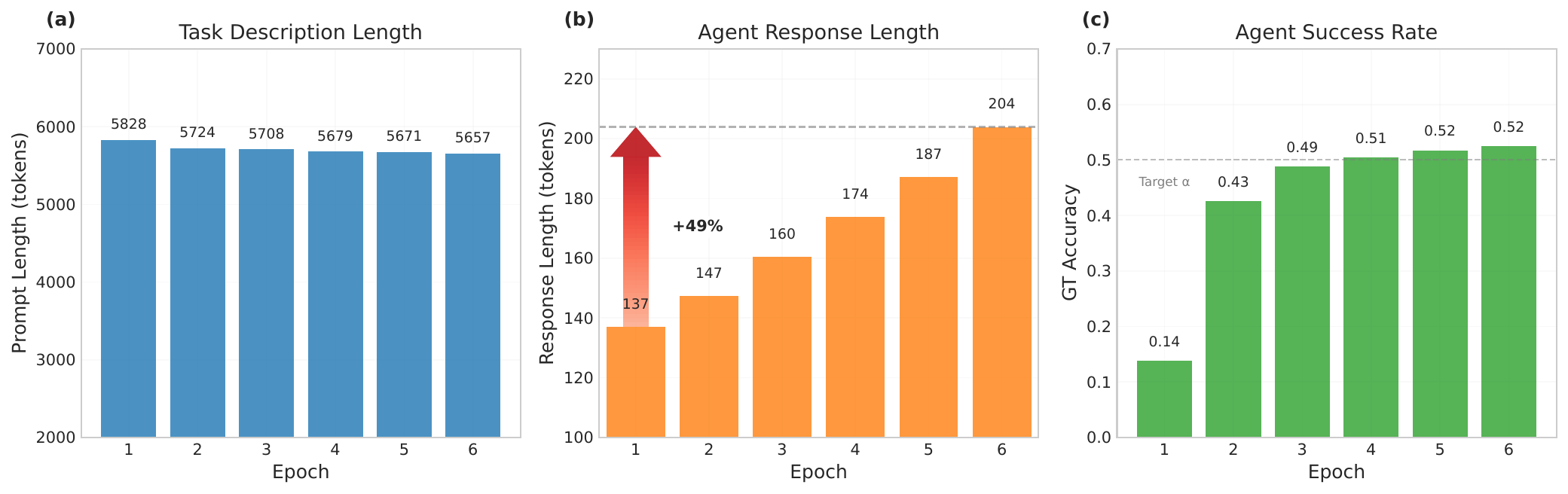}
\caption{\textbf{Emergent curriculum in \autoenv.} 
Across training epochs, the environment simulator gradually increases task complexity (a), reflected by longer task descriptions; 
the agent correspondingly produces longer reasoning chains (b) as it learns to solve harder tasks; 
and its success rate (c) remains within a controlled band despite rising difficulty. 
Together these curves show that \autoenv induces an emergent curriculum in which task difficulty and agent capability co-evolve in a stable manner.}
\label{fig:curriculum_analysis}
\end{figure}
\paragraph{Training Configuration.}
For the Agent Policy ($\pi_{\text{agent}}$), we train for 10 epochs using GRPO with batch size 64 and maximum sequence length 9{,}000 tokens (prompt + response). The Environment Policy ($\pi_{\text{env}}$) is updated at the same epoch frequency via RWR with batch size 64. We use the same optimizer family (AdamW) for both policies, with learning rates detailed in Appendix~\ref{sec:hyperparams}. All methods are trained on the same base dataset, and for \autoenv variants the additional data comes solely from the simulator.

\paragraph{Evaluation Protocol.}
All models---including large models---are evaluated in a unified tool-calling framework. We use identical system prompts, tool specifications, and decoding settings for all models on a given benchmark. We do not attach additional multi-agent orchestration or human-in-the-loop corrections to any method, to focus the comparison on training and data regimes.

\paragraph{Summary of Results.} Across all benchmarks, \autoenv improves the 7B base agent significantly, with gains up to \textbf{+40.3\%} on ALFWorld and \textbf{+20.4\%} on API-Bank compared to strong baselines.

\subsection{RQ1: Does \autoenv Improve Downstream Task Performance?}
Table~\ref{tab:results} presents the main comparison. \autoenv consistently outperforms both general-purpose models and specialized RL agents across all five benchmarks. Notably, on \textbf{ALFWorld}, which requires long-horizon planning, \autoenv achieves 54.5\% accuracy compared to 14.2\% for the base model, demonstrating the power of the simulator in generating diverse embodied scenarios that are costly to collect in the real world. On API-Bank and BFCL, \autoenv achieves 79.1\% and 41.8\% success respectively, markedly improving over other 7B baselines that rely on static data or non-adaptive exploration.

Beyond absolute performance, \autoenv also closes much of the gap to substantially larger models. The average score of \autoenv (53.6) is competitive with---and in many cases exceeds---that of 14B–72B models that do not benefit from a difficulty-aligned simulator. This supports our central claim that \emph{how} data is generated and aligned with the agent matters as much as, or more than, simply scaling model size or collecting larger static datasets. In the remainder of the section, we investigate how the co-evolutionary process shapes the curriculum, data efficiency, and difficulty calibration of the environment.

We also verify that the co-evolutionary training process itself is well behaved. Figure~\ref{fig:training_stability} plots the training dynamics of \autoenv: the per-step GRPO reward and batch-level ground-truth accuracy both increase steadily, and the validation score improves monotonically before saturating, without signs of reward hacking or instability. This suggests that our difficulty-aligned simulator can be optimized jointly with the agent using standard policy gradients, without introducing pathological oscillations during training.

\subsection{RQ2: Does \autoenv Learn to Tackle Harder Tasks Over Time?}
We next investigate whether the environment simulator actually learns a curriculum, rather than simply generating random variations. To this end, we use the average length of the agent's required response as a proxy for reasoning complexity and task difficulty, and track how it evolves throughout training. Figure~\ref{fig:curriculum_analysis} plots this average response length together with the agent's success rate on simulated tasks across epochs.

\finding{The required response length increases from 137 to 204 tokens (+49\%) by epoch 6. This confirms that $\pi_{\text{env}}$ learns to generate progressively more complex reasoning challenges as $\pi_{\text{agent}}$ becomes more capable, creating an emergent curriculum without manual design. At the same time, the agent's success rate does not collapse; instead, it grows in tandem with task complexity, suggesting that the simulator is adapting difficulty in a controlled way rather than simply making tasks arbitrarily harder.}

This pattern is consistent with our design of the \alphacurr: as the agent improves on a given family of tasks, their success rate on that family moves away from the target band around $\alpha$, reducing its contribution to $R_{\text{env}}$ and encouraging the simulator to propose harder variations. The observed increase in response length indicates that these harder tasks require more extensive multi-step reasoning and tool use, rather than superficial changes to surface form. Qualitatively, we observe that later tasks tend to involve more complex compositions of tools and deeper chains of intermediate subgoals.

Finally, the fact that curriculum emerges \emph{without any hand-specified difficulty schedule} supports our broader view of \autoenv as a data-evolving system: the environment learns where the agent's ``breaking points'' are and adapts task generation accordingly. This contrasts with conventional curriculum learning, which typically relies on fixed heuristics or manually designed difficulty levels. Here, difficulty is inferred directly from the agent's behaviour via $R_{\text{env}}(\hat{p})$.

\subsection{RQ3: Is \autoenv More Data-Efficient Than Gemini-Based Augmentation?}
\begin{figure}[h]
\centering
\begin{subfigure}{0.48\textwidth}
    \centering
    \includegraphics[width=\linewidth]{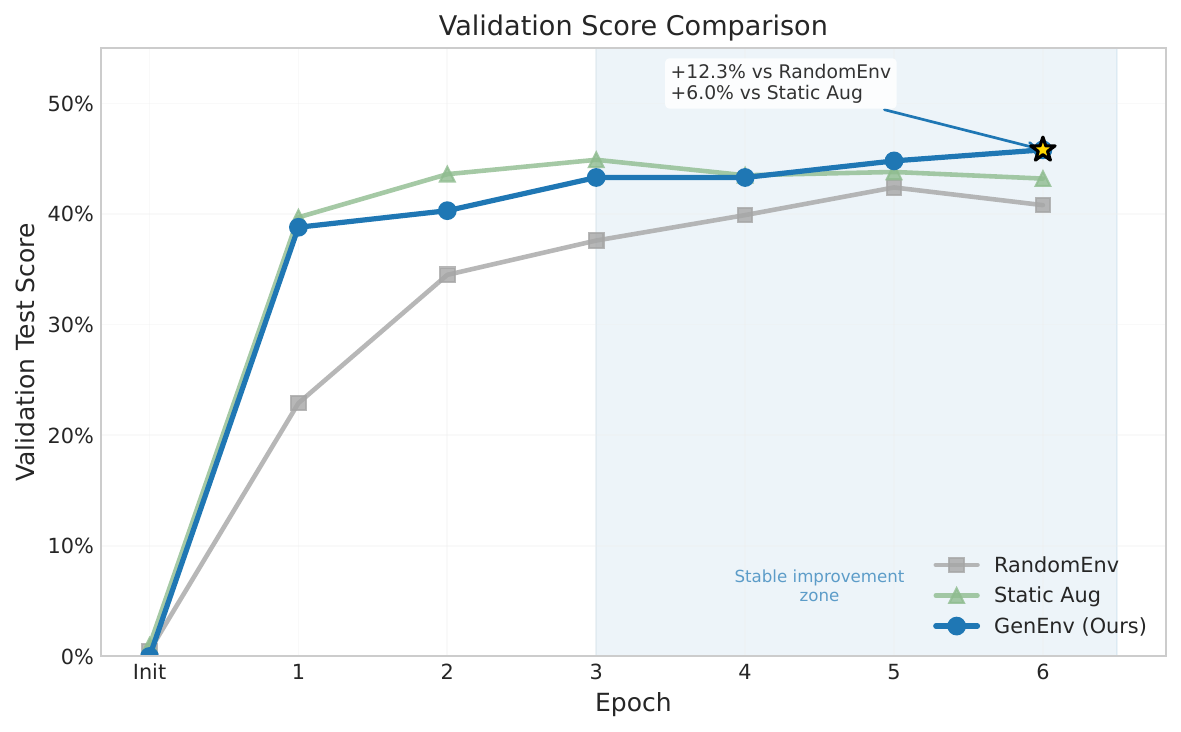}
    \caption{\textbf{Method comparison.} \autoenv outperforms both static Gemini-based augmentation and the GenEnv-Random variant, showing that \emph{how} data is generated and aligned with the agent matters more than simply adding more offline synthetic data.}
    \label{fig:method_comparison}
\end{subfigure}
\hfill
\begin{subfigure}{0.48\textwidth}
    \centering
    \includegraphics[width=\linewidth]{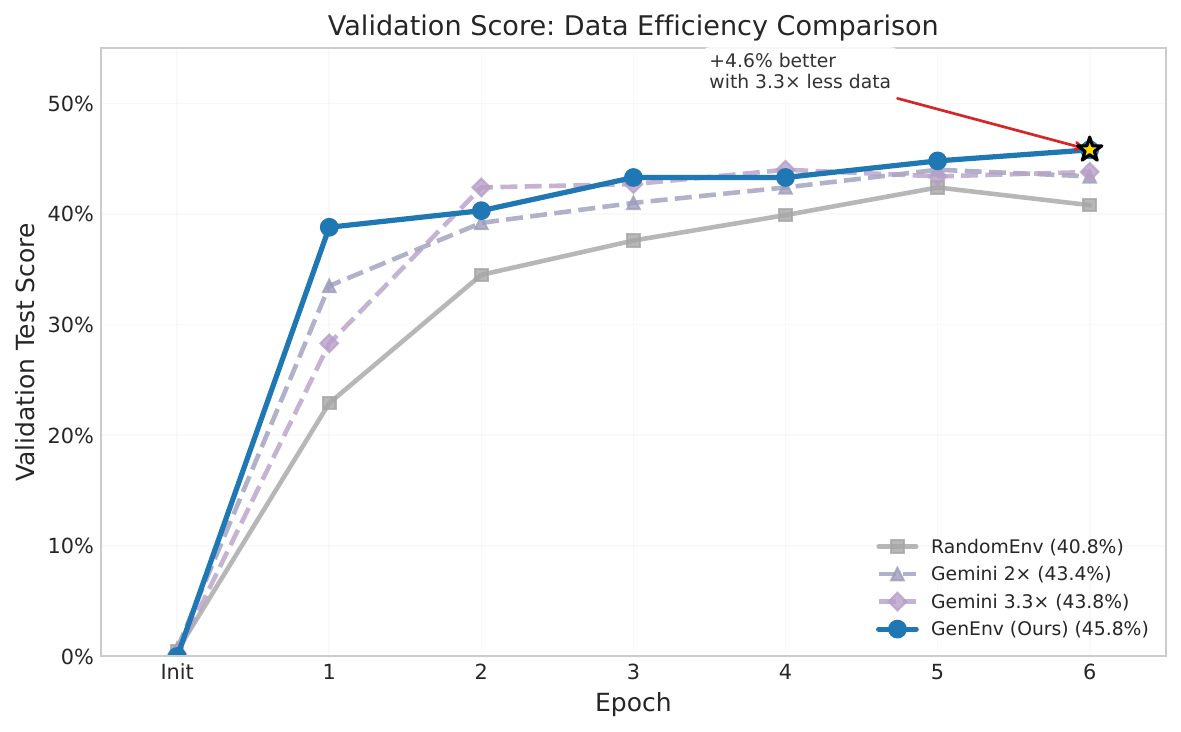}
    \caption{\textbf{Data efficiency.} \autoenv achieves higher validation performance while using substantially fewer synthetic samples than Gemini-based offline augmentation, indicating that difficulty-aligned, on-policy simulation provides more learning signal per example than untargeted teacher-generated data.}
    \label{fig:data_efficiency}
\end{subfigure}
\caption{\textbf{Static vs.\ difficulty-aligned simulation.}}
\label{fig:method_and_efficiency}
\end{figure}

\noindent A key question is whether \autoenv is simply benefiting from ``more data'' or from \emph{better-targeted} data. We compare \autoenv (which generates data on-the-fly) against offline augmentation using Gemini 2.5 Pro, under comparable training budgets. Gemini-Offline (2x / 3.3x) corresponds to large static corpora generated before training, while \autoenv continuously adapts task generation as the agent evolves.

\finding{As shown in Figure~\ref{fig:data_efficiency}, \autoenv (using 1x original data + dynamic simulation) reaches a validation score of 0.458. This outperforms Gemini-Offline (3.3x) (0.438), which uses $\approx$3.3x more synthetic data generated by a much stronger model. This confirms that \textit{targeted, difficulty-aligned} data generation can be structurally superior to massive but untargeted augmentation. Furthermore, \autoenv outperforms GenEnv-Random by 12.3\%, indicating that the $R_{\text{env}}$ optimization is critical.}

These results highlight two distinct effects. First, merely adding more synthetic trajectories from a powerful teacher model quickly encounters diminishing returns: once the static dataset ceases to match the agent's current weaknesses, additional examples provide limited new learning signal. Second, the comparison between \autoenv and GenEnv-Random controls for the presence of a simulator: both generate trajectories online, but only \autoenv trains the simulator to target the $\alpha$ band. The performance gap between these two variants isolates the benefit of difficulty alignment itself, rather than just the benefit of having a generative environment.

From a practical standpoint, these findings suggest a shift in how we invest computational and annotation budget. Instead of paying for ever larger static datasets created by stronger teachers, it may be more effective to invest in a moderately sized simulator that co-evolves with the student agent. This is especially attractive in domains where collecting real trajectories is expensive or slow, as the simulator can keep generating fresh, on-policy data without requiring repeated human involvement.

\subsection{RQ4: Does the Environment Reward Produce Well-Calibrated Difficulty?}
\begin{figure}[h]
\centering
\includegraphics[width=0.7\textwidth]{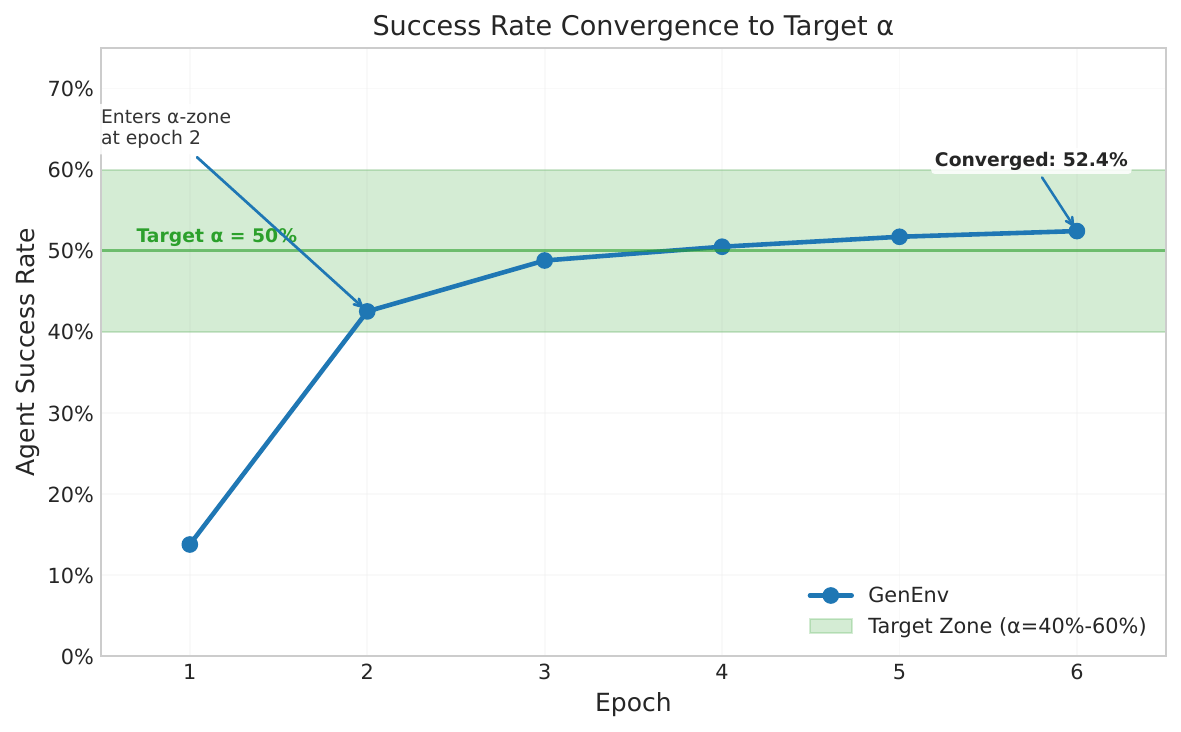}
\caption{\textbf{Difficulty calibration via the $\alpha$-Curriculum Reward.} 
As training progresses, the agent’s success rate on simulator-generated tasks converges to the target difficulty band 
(centered at $\alpha = 0.5$), demonstrating that the environment policy reliably adapts task difficulty to match the agent’s current capability. 
This empirically verifies the theoretical ranking consistency of $R_{\text{env}}$ and shows that the simulator self-calibrates to maintain tasks in the zone of proximal development.}
\label{fig:difficulty_convergence}
\end{figure}

\begin{figure}[h]
\centering
\includegraphics[width=\textwidth]{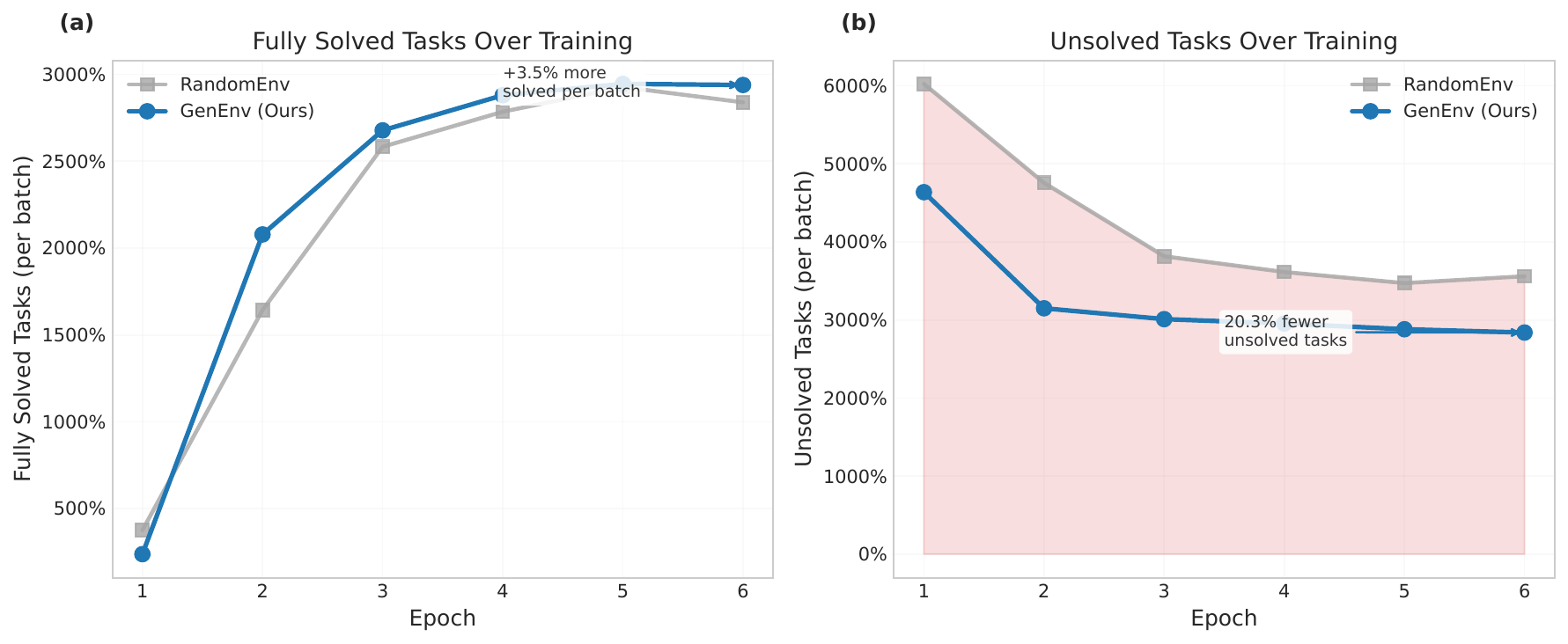}
\caption{\textbf{Problem-solving behavior during training.} 
(a) \autoenv consistently increases the proportion of fully solved tasks per batch, surpassing the RandomEnv variant; 
(b) the rate of unsolved tasks decreases substantially faster under \autoenv. 
These trends show that difficulty-aligned simulation not only improves average performance but also accelerates the elimination of failure modes compared to unguided task generation.}
\label{fig:solve_analysis}
\end{figure}

\noindent Finally, we verify if the \alphacurr successfully calibrates task difficulty. During training, we track the agent's success rate on simulated tasks generated by $\pi_{\text{env}}$ and examine whether it converges to the intended band around $\alpha$. Figure~\ref{fig:difficulty_convergence} summarizes this trajectory over training epochs.

\finding{Figure~\ref{fig:difficulty_convergence} shows the agent's success rate on generated tasks during training. Starting from 0.138, the success rate converges towards a band centered at the target difficulty $\alpha = 0.5$, remaining within a range of approximately $[0.4, 0.6]$ for most of training. This suggests that $\pi_{\text{env}}$ is actively optimizing for the ``zone of proximal development,'' avoiding the collapse into trivial or impossible tasks that plagues random generation.}

This behaviour is precisely what our theoretical analysis predicts. Theorem~\ref{thm:ranking} shows that, given enough rollouts, $R_{\text{env}}(\hat{p})$ provides a consistent ranking signal that favours task types with success probabilities closest to $\alpha$. Empirically, we see this mechanism in action: early in training, when most tasks are either too hard or too easy, the simulator updates quickly reshape the distribution towards intermediate difficulty. Later, updates become smaller as the success rate stabilizes near the target band, leading to a self-calibrated curriculum.

We also observe that the calibrated difficulty band coexists with improved downstream performance on real benchmarks. That is, the simulator does not merely ``keep the agent at 50\% success'' on synthetic tasks; rather, it continually moves the frontier of what intermediate difficulty means as the agent learns. This reinforces our view of \autoenv as a genuinely \emph{co-evolving} system, in which both the agent and the task distribution adapt in lockstep.

\section{Related Work}
\label{sec:related_work}

\subsection{Large Language Model Agents}
Recent advancements have demonstrated the ability of LLMs to function as autonomous agents. Pioneering works like ReAct, Reflexion, and Voyager have shown that combining chain-of-thought reasoning with action generation and memory evaluation enables agents to tackle complex tasks \citep{yao2022, shinn2023, wang2023, zou2025latent, guo2025role}. More recent efforts such as KnowAgent and MemBench further refine planning and memory capabilities in agentic settings \citep{zhang2024knowagent, tan2025membench}. Others, such as Toolformer and WebGPT, have focused on augmenting LLMs with external tools and web-browsing capabilities, expanding their operational scope \citep{schick2023, nakano2021}. While these models showcase strong performance, their training paradigms primarily rely on imitation learning from static, pre-collected datasets of expert trajectories \citep{nakano2021, wang2023}. This dependency on expert data forms a significant bottleneck, limiting the agent's ability to explore and discover strategies beyond the provided demonstrations \citep{shinn2023}. Our work addresses this limitation by creating a dynamic learning environment that does not solely depend on a fixed dataset.

\subsection{Trajectory Synthesis for Agent Training}
Recent efforts have focused on \textbf{generating synthetic trajectories} to address the limitations of fixed expert data for training LLM agents \citep{yu2025demystifying}. Some methods focus on \textbf{offline synthetic data generation}, creating novel trajectories to increase diversity and coverage where expert data is sparse \citep{ye2024, ding2024}. Other approaches leverage LLMs to generate self-reflective trajectories, incorporating reflections and corrections to learn from errors \citep{chen2025}, or to provide stepwise guidance from a teacher model toward correcting mistakes \citep{chen2025stepwise}. For web agents, scalable synthesis pipelines use web tutorials or exploration-driven methods to produce large-scale \textbf{synthetic datasets} of multimodal trajectories \citep{pahuja2025, yuan2024}. Other works introduce iterative self-training for reflection \citep{yuan2025, wang2025co}, fine-tuning on massive interaction trajectories \citep{zhang2025}, step-level calibration \citep{luo2025}, and simulators for online exploration to generate high-quality feedback-driven data \citep{hoang2025}.

Beyond these, several recent frameworks further advance autonomous and adaptive trajectory synthesis. \citet{wang2025uisimulator} propose a scalable LLM-based digital environment that models user-interface transitions as structured trajectories, combining simulation and targeted scaling to expose agents to increasingly complex states. \citet{zhao2025autoplay} build upon exploration-driven task generation, using a two-stage pipeline that first explores application environments and then synthesizes executable trajectories, yielding tens of thousands of realistic multimodal interaction traces. In the data-science domain, \citet{zhang2025deepanalyze} integrate curriculum-based agentic training with a data-grounded trajectory synthesis framework to produce high-fidelity analytical workflows, allowing smaller models to outperform larger workflow-based agents. Complementarily, \citet{liu2025limi} challenge the assumption that more data is always better---demonstrating that strategically curated and high-quality trajectories can induce stronger agentic reasoning from only a handful of demonstrations. \citet{chen2025compo} also demonstrates that noisy preference data could be utilized into improve agent alignment. Finally, \citet{sun2025earlyexperience} introduce an implicit-feedback paradigm where agents learn from their own early interactions before formal reinforcement learning, leveraging self-reflection and world-modeling to bootstrap generalization from suboptimal actions.

Taken together, these developments indicate a trend toward \textbf{adaptive, self-improving trajectory synthesis}: instead of relying on static or expert-curated data, agents increasingly generate, evaluate, and refine their own experiences---closing the loop between simulation, exploration, and learning. This motivates our approach, which leverages a dynamic environment simulator to generate these adaptive experiences on-demand, directly addressing the limitations of static datasets.

\subsection{Environment Simulation}
Simulators have long been a cornerstone of reinforcement learning, especially in domains such as robotics where interacting with the real world is costly \citep{todorov2012}. In the context of LLM agents, environment simulation has emerged as a key mechanism for generating training data and evaluating agent capabilities. \citet{wang2025uisimulator} demonstrate how an LLM-powered digital world simulator can generate structured user-interface states and transitions; its targeted scaling strategy produces diverse, high-impact tasks and yields agents that rival those trained on real UIs. Complementary simulation frameworks move beyond UI tasks: \citet{zhou2025faithfulsim} introduce a scalable closed-loop simulator that samples multi-step tasks from a tool--relationship graph, simulates interactions with configurable user and environment archetypes, and evaluates procedural alignment and success. Experiments reveal that environment reliability and user archetypes are dominant factors in agent performance. Beyond task-centric environments, \citet{li2025agentsociety} integrate LLM-driven agents with a realistic societal environment and a large-scale simulation engine, generating social lives for over ten thousand agents and millions of interactions among agents and their surroundings. These works highlight the importance of faithful and diverse environment simulation in creating high-quality training data and benchmarks.

Our environment LLM diverges from traditional simulators: rather than predicting state transitions or user responses for a fixed task, it generates entire tasks and goals \emph{conditioned on the agent’s recent performance}, with an explicit objective to match a target difficulty band. This casts environment design itself as a learnable policy with its own reward signal.

\section{Conclusion}
We presented \textbf{\autoenv}, a framework that shifts agent training from a static, model-evolving process to a dynamic, \textbf{data-evolving game}. By establishing a difficulty-aligned co-evolutionary loop between an Agent Policy and an Environment Policy, \autoenv achieves superior performance and data efficiency on a diverse suite of agent benchmarks. Our results suggest that future agent training systems should move beyond larger static datasets toward adaptive, self-calibrating simulation environments. Beyond the particular instantiation studied here, we believe that difficulty-aligned simulators can serve as a general recipe for training robust LLM agents in domains where real-world exploration is costly or risky.
\newpage
\bibliography{iclr2026_conference}

\appendix
\newpage
\section{Appendix}
\subsection{Proofs for Section~\ref{sec:theory}}
\label{app:theory-proofs}

In this appendix we provide detailed proofs for the theoretical results
stated in Section~\ref{sec:theory}.

\subsubsection{Proof of Proposition~\ref{prop:info}}

Recall that for a fixed task type $\tau$ and parameter vector $\theta$,
the Agent Policy update uses the REINFORCE-style estimator
\[
    g(\tau, r)
    = (r - b(\tau)) \nabla_\theta \log \pi_\theta(a \mid \tau),
\]
and we choose the baseline to be
$b(\tau) = \mathbb{E}[r \mid \tau] = p(\tau)$.
Let $S = \nabla_\theta \log \pi_\theta(a \mid \tau)$ for brevity.
Conditioned on $\tau$, the reward $r$ is Bernoulli with success
probability $p(\tau)$, so $\mathbb{E}[r \mid \tau] = p(\tau)$ and
$\mathrm{Var}(r \mid \tau) = p(\tau)(1 - p(\tau))$.

We have
\begin{align}
    \mathbb{E}\big[\|g(\tau, r)\|^2 \mid \tau\big]
    &= \mathbb{E}\big[(r - p(\tau))^2 \, \|S\|^2 \mid \tau\big] \\
    &= \sum_{r \in \{0,1\}} \Pr(r \mid \tau)
       \, (r - p(\tau))^2 \,
       \mathbb{E}\big[\|S\|^2 \mid \tau, r\big].
\end{align}
By Assumption~\ref{assump:bounded-score}, for both $r=0$ and $r=1$ we
have
\[
    c_{\min} \le \mathbb{E}\big[\|S\|^2 \mid \tau, r\big]
    \le c_{\max}.
\]
Therefore
\begin{align}
    \mathbb{E}\big[\|g(\tau, r)\|^2 \mid \tau\big]
    &\ge c_{\min}
      \sum_{r \in \{0,1\}} \Pr(r \mid \tau) (r - p(\tau))^2 \\
    &= c_{\min} \, \mathbb{E}\big[(r - p(\tau))^2 \mid \tau\big] \\
    &= c_{\min} \, \mathrm{Var}(r \mid \tau) \\
    &= c_{\min} \, p(\tau)\big(1 - p(\tau)\big),
\end{align}
and similarly
\begin{align}
    \mathbb{E}\big[\|g(\tau, r)\|^2 \mid \tau\big]
    &\le c_{\max}
      \sum_{r \in \{0,1\}} \Pr(r \mid \tau) (r - p(\tau))^2 \\
    &= c_{\max} \, p(\tau)\big(1 - p(\tau)\big).
\end{align}
Taking $C_{\min} = c_{\min}$ and $C_{\max} = c_{\max}$ yields the
bounds in Equation~\eqref{eq:variance-bound}. Since
$p(1-p)$ is a concave quadratic on $[0,1]$ with a unique maximum at
$p=1/2$, the expected squared gradient norm is maximized (up to
constant factors) for tasks with $p(\tau) = 1/2$, i.e., tasks of
intermediate difficulty. This proves
Proposition~\ref{prop:info}.

\subsubsection{Proof of Theorem~\ref{thm:ranking}}

We restate the setting for clarity. For $i \in \{1,2\}$, the true
success probability on task type $\tau_i$ is $p_i$, and we define
$\Delta_i = |p_i - \alpha|$ with $\Delta_1 < \Delta_2$.
From $n_i$ independent rollouts we obtain the empirical success rate
$\hat{p}_i = k_i / n_i$, where $k_i$ is the number of successes.
The $\alpha$-Curriculum Reward is
\[
    R_{\text{env}}(\hat{p}_i)
    = \exp\big(-\beta(\hat{p}_i - \alpha)^2\big).
\]

Since the exponential function is strictly monotone decreasing in
$(\hat{p}_i - \alpha)^2$, we have
\[
    R_{\text{env}}(\hat{p}_1) \le R_{\text{env}}(\hat{p}_2)
    \quad\Longleftrightarrow\quad
    |\hat{p}_1 - \alpha| \ge |\hat{p}_2 - \alpha|.
\]
Thus the event that the Environment Policy mis-ranks the two tasks
(i.e., gives $\tau_1$ no larger reward than $\tau_2$) is exactly the
event
\[
    \mathcal{E}_{\text{mis}} =
    \{\, |\hat{p}_1 - \alpha| \ge |\hat{p}_2 - \alpha| \,\}.
\]

Let $n = \min\{n_1, n_2\}$ and define
\[
    \delta = \frac{\Delta_2 - \Delta_1}{3} > 0.
\]
Consider the event
\[
    \mathcal{E}_{\text{good}} =
    \big\{\, |\hat{p}_1 - p_1| \le \delta
            \text{ and }
            |\hat{p}_2 - p_2| \le \delta
    \,\big\}.
\]
We claim that on $\mathcal{E}_{\text{good}}$ we must have
$|\hat{p}_1 - \alpha| < |\hat{p}_2 - \alpha|$, and hence
$R_{\text{env}}(\hat{p}_1) > R_{\text{env}}(\hat{p}_2)$.
Indeed, by the triangle inequality,
\begin{align}
    |\hat{p}_1 - \alpha|
    &\le |p_1 - \alpha| + |\hat{p}_1 - p_1|
     \le \Delta_1 + \delta, \\
    |\hat{p}_2 - \alpha|
    &\ge |p_2 - \alpha| - |\hat{p}_2 - p_2|
     \ge \Delta_2 - \delta.
\end{align}
By the choice of $\delta$, we have
\[
    \Delta_1 + \delta
    = \Delta_1 + \frac{\Delta_2 - \Delta_1}{3}
    = \frac{2\Delta_1 + \Delta_2}{3},
    \qquad
    \Delta_2 - \delta
    = \Delta_2 - \frac{\Delta_2 - \Delta_1}{3}
    = \frac{\Delta_1 + 2\Delta_2}{3},
\]
and since $\Delta_1 < \Delta_2$, it follows that
\[
    \Delta_1 + \delta
    = \frac{2\Delta_1 + \Delta_2}{3}
    < \frac{\Delta_1 + 2\Delta_2}{3}
    = \Delta_2 - \delta.
\]
Therefore,
\[
    |\hat{p}_1 - \alpha|
    \le \Delta_1 + \delta
    < \Delta_2 - \delta
    \le |\hat{p}_2 - \alpha|
\]
on $\mathcal{E}_{\text{good}}$, which implies
$R_{\text{env}}(\hat{p}_1) > R_{\text{env}}(\hat{p}_2)$.
Consequently, $\mathcal{E}_{\text{mis}}$ can only occur on the
complement event $\mathcal{E}_{\text{good}}^{c}$, and hence
\[
    \Pr(\mathcal{E}_{\text{mis}})
    \le \Pr(\mathcal{E}_{\text{good}}^{c})
    \le \Pr\big(|\hat{p}_1 - p_1| > \delta\big)
       + \Pr\big(|\hat{p}_2 - p_2| > \delta\big),
\]
where the last inequality is a union bound.

For each $i \in \{1,2\}$, $\hat{p}_i$ is the empirical mean of $n_i$
i.i.d.\ Bernoulli random variables with mean $p_i$. By Hoeffding's
inequality,
\[
    \Pr\big(|\hat{p}_i - p_i| > \delta\big)
    \le 2 \exp(-2 n_i \delta^2)
    \le 2 \exp(-2 n \delta^2),
\]
where $n = \min\{n_1, n_2\}$. Therefore
\begin{align}
    \Pr(\mathcal{E}_{\text{mis}})
    &\le 4 \exp(-2 n \delta^2) \\
    &= 4 \exp\!\left(
        - 2 n \left(\frac{\Delta_2 - \Delta_1}{3}\right)^2
    \right) \\
    &= 4 \exp\!\left(
        - \frac{2}{9} (\Delta_2 - \Delta_1)^2 \, n
    \right),
\end{align}
which establishes the desired exponential bound and completes the
proof of Theorem~\ref{thm:ranking}.

\subsection{Hyperparameter Details}
\label{sec:hyperparams}
Table~\ref{tab:hyperparams} lists the key hyperparameters used in our experiments. Note that for the Agent ($\pi_{\text{agent}}$), we employ GRPO, while the Environment ($\pi_{\text{env}}$) uses Reward-Weighted Regression (RWR).

\begin{table}[h!]
\centering
\caption{Hyperparameters for \autoenv training.}
\label{tab:hyperparams}
\begin{tabular}{lcc}
\toprule
\textbf{Parameter} & \textbf{Agent Policy ($\pi_{\text{agent}}$)} & \textbf{Environment Policy ($\pi_{\text{env}}$)} \\
\midrule
\multicolumn{3}{c}{\textbf{Optimizer (AdamW)}} \\
Learning Rate & $1 \times 10^{-6}$ & $5 \times 10^{-7}$ \\
Batch Size & 64 & 64 \\
\midrule
\multicolumn{3}{c}{\textbf{Policy Optimization}} \\
Method & GRPO & Reward-Weighted Update (RWR) \\
Total Epochs & 10 & 10 \\
Target Difficulty $\alpha$ & N/A & 0.5 \\
Difficulty Filter $k_{\text{min}}$ & N/A & 0.1 \\
RWR Temperature $\lambda$ & N/A & 1.0 \\
\bottomrule
\end{tabular}
\end{table}

\subsection{Environment Baseline Implementation Details}
Here we detail the specific implementation of the environment variants used in Section \ref{sec:experiments}:
\begin{itemize}
    \item \textbf{GenEnv-Random:} We use the same base model (Qwen2.5-7B-Instruct) for the environment. The \texttt{autoenv.disable\_env\_training} flag is set to \texttt{True}. It generates 4 variations per prompt per epoch dynamically, but the model weights are never updated based on $R_{\text{env}}(\hat{p})$.
    \item \textbf{GenEnv-Static:} We use the \texttt{generate\_static\_augmentation.py} script to pre-generate 5 variations for each of the 544 original training samples, resulting in a fixed dataset of 3{,}264 samples. The agent is trained using standard PPO for 10 epochs.
    \item \textbf{Gemini-Offline Baselines:} We prompted Gemini 2.5 Pro to generate variations of the training data. Due to API constraints and filtering, the \textbf{Gemini-2x} setting resulted in 957 samples ($\approx$1.76x) and \textbf{Gemini-4x} resulted in 1{,}777 samples ($\approx$3.27x). These represent high-quality, but static, external data augmentation.
\end{itemize}

\end{document}